\documentclass{article}

\usepackage{microtype}
\usepackage{graphicx}
\usepackage{subfigure}
\usepackage{booktabs} % for professional tables

\usepackage[utf8]{inputenc}
\usepackage{amsmath}
\usepackage{amssymb}
\usepackage{amsthm}
\usepackage{xcolor}
\usepackage[normalem]{ulem}
\usepackage{wrapfig}

\usepackage{algorithm}
\usepackage[noend]{algpseudocode}

\usepackage[pdftex,colorlinks,linkcolor=blue,citecolor=blue,filecolor=blue,urlcolor=blue]{hyperref}
\usepackage[letterpaper,hmargin=1.0in,vmargin=1.0in]{geometry}
\usepackage[margin=20pt,font=small,labelfont=bf]{caption}
\usepackage{nicefrac}
\newcommand{\hh}{HNHN }

\newtheorem{theo}{Theorem}[section]

\newtheorem{lemma}[theo]{Lemma}

\theoremstyle{remark}

\theoremstyle{definition}

\title{HNHN: Hypergraph Networks with Hyperedge Neurons}
\author{
  Yihe Dong\\
  Microsoft \\
  \and
  Will Sawin \\
  Department of Mathematics \\
  Columbia University \\
  \and
  Yoshua Bengio \\
  Mila \\ 
  Universit\'{e} de Montr\'{e}al %é de Montréal
}

\begin{document}

\date{}
\maketitle
\begin{abstract}
Hypergraphs provide a natural representation for many real world datasets. We propose a novel framework, HNHN, for hypergraph representation learning. HNHN is a hypergraph convolution network with nonlinear activation functions applied to both hypernodes and hyperedges, combined with a normalization scheme that can flexibly adjust the importance of high-cardinality hyperedges and high-degree vertices depending on the dataset. We demonstrate improved performance of HNHN in both classification accuracy and speed on real world datasets when compared to state of the art methods.

\end{abstract}

\section{Introduction}

Much real world data can be represented as graphs, as data come in the form of objects and relations between the objects. In many cases, a relation can connect more than two objects. %i.e. edges that connect more than two vertices, 
A hypergraph can naturally represent such structures. Take for instance paper-authorship graphs, where the vertices represent papers. Since one author can write multiple papers, it is natural to represent authors as hyperedges in a hypergraph. %\todo{hnhn stands for}

Our goal is to learn representations of such structured data with a novel hypergraph convolution algorithm.
%\todo{non linear relations?} In particular, we aim to rectify the shortcoming in existing graph-based approaches that, in the real world, there often exist relations connecting a set of nodes, with cardinality $>2$. We do this by using hypergraphs, where each edge can connect an arbitrary number of nodes. 
First let's recall the vanilla graph neural network (GNN): fix a graph $G$ with $n$ vertices, and let $A \in \mathbb R^{n\times n}$ be its adjacency matrix. Choosing a suitable $d$, we define node representations $X \in \mathbb R^{n \times d} $, as well as weights $W \in \mathbb R^{d\times d} $ and bias $b\in \mathbb R^d$, and a nonlinear activation function $\sigma$. In each layer of the GNN, the node representations are updated as: 
\begin{equation}
\label{graph-convolution} X' = \sigma(AXW + b).
\end{equation}

Some past research has generalized this to hypergraphs by defining an appropriate hypergraph analogue of the adjacency matrix $A$ \cite{hgnn,hcha}. For vertices $i$ and $j$, the entry $A_{ij}$ can be defined as a sum over hyperedges containing both $i$ and $j$ of a weight function, that may depend on the number of vertices in the hyperedge. However, these approaches do not fully utilize the hypergraph structure. In fact, the matrix $A$ defined this way can be seen to equal the adjacency matrix of the graph obtained by replacing each hyperedge with a (weighted) clique. Thus, the hypergraph algorithm will have no better accuracy on a given task than the corresponding graph algorithm, applied to this graph built from cliques.

\noindent{\bf Hyperedge nonlinearity.} To solve this problem, we treat hyperedges as objects worthy of study in their own right. We therefore train a network with one representation $X_V$ for hypernodes and another representation $X_E$ for hyperedges. We can then use the hypernode-hyperedge incidence matrix for the convolution step. A nonlinear function $\sigma$ then acts on both hypernodes and hyperedges. By making this change, we allow the network to learn nonlinear behavior of individual hyperedges, which plausibly exists in many real-world data sets. For instance, the probability that one author of a paper with unknown research interests works in a particular research area might be a nonlinear function of the number of other authors of the paper that work in that research area.

\noindent {\bf Normalization.} In addition, we take a different, more flexible approach to normalization. Normalization in graph convolution is required because formula \eqref{graph-convolution} involves pooling the representations of all vertices adjacent to a given vertex. This produces large values for vertices of high degrees but small values for vertices of low degrees. Normalization ensures numerical stability during training. Different options exist for normalization, depending on how we weight the different vertices (as a function of their degrees). In the hypergraph setting, there is an analogous choice of how to weight different hypernodes and hyperedges. We allow these choices to depend on hyperparameters that are optimized for a given data set.

Indeed, there is little reason to suspect that the most mathematically elegant normalization formula, for instance using the symmetric normalized Laplacian, $D^{-1/2} L D^{-1/2}$, necessarily gives the best accuracy for prediction tasks. Instead, normalization represents a choice of how to weight vertices of large degree compared to vertices of smaller degree. This choice should be made according to whether weighting vertices of large degree highly gives better accuracy on a particular dataest - e.g. on one social network, users with a large number of connections may be important and influential individuals whose activity is useful for prediction, while on another, they may be bots who make connections at random.

Our key contributions in this paper consist of introducing a hypergraph convolution network with nonlinear activation functions applied to both hypernodes and hyperedges, and a more flexible approach to normalization.

\noindent{\bf Paper Overview.} This paper is organized as follows: In \S\ref{sec-model}, we begin by describing our hypergraph convolution method, \hh (Hypergraph Network with Hyperedge Neurons), without the normalization step, to focus on the structure of the convolution. We analyze how \hh relates to, and differs from, approaches that apply graph convolution to hypergraphs by using the clique or star expansion to transform hypergraphs into graphs. We then explain how \hh relates to hypergraph convolution methods proposed in prior works. Finally, we introduce the normalization step. In \S\ref{sec-experiments}, we empirically study aspects of HNHN, from classification accuracy on real datasets to dependence on hyperparameters.

%In addition, we introduce a normalization scheme where hyperedges with high degrees are given increased weight (or reduced weight) during message passing depending on a tunable hyperparameter. This allows the model to better reflect the importance of large and small hyperedges in different datasets, leading to substantially improved performance.

%One \textbf{challenge} will be to find or generate hypergraphs based on real world datasets, for language, we can use a parser, but parsers often make mistakes in interpreting moderate to complex relations between tokens. For images, we can likely derive hypergraphs from scene graphs in datasets such as Visual Genome. 

%Yihe will do some literature search to see to what extent this has been done.

\section{Model architecture and analysis}
\label{sec-model}
%Another example is knowledge graphs, where .
%Real world data naturally come in     with relations naturally amenable to a hypergraph, for instance a cat is a mammal, which is an animal, here the mammal class is an hyperedge. %%but also a vertex!
%Thus the   naturally follows a bipartite graph structure. 

First some \textbf{notation}: we define a hypergraph $H$ to consist of a set $V$ of hypernodes and set $E$ of hyperedges, where each hyperedge is itself a set of hypernodes. Let $n= |V|$ and $m=|E|$. Indexing the hypernodes as $v_i$ for $i\in \{1, ..., n\}$, and the hyperedges as $v_j$ for $j \in \{1, ..., m\}$, we define the \emph{incidence matrix} $A \in \mathbb R^{n \times m}$ by $$A_{ij} = \begin{cases} 1 & v_i \in e_j \\ 0 & v_i \notin e_j \end{cases} $$

We then update the hypernode representations $X_V \in \mathbb R^{n \times d} $ and hyperedge representations $X_E \in \mathbb R^{m\times d} $ by a convolution using the incidence matrix $A$. Thus, our update rule is given by the formulas
%hypernodes $\{v\}_i$ and hyperedges $\{e\}_j$, we create an associated \textit{bipartite graph}, \bh{H}, whose left hand side nodes, $X_V$, represent $\{v\}_i$, and the right hand side nodes, $X_E$, represent $\{e\}_j$. %, which we also refer to as factors.
%This is a very natural way to create a non-hyper graph out of a hypergraph, while retaining the hypergraph's structure.
%In particular, an edge in $B$ between $X_V \in X_V$ and $X_E \in X_E$ indicates the node in $H$ corresponding to $x_1$ is connected via the hyperedge in $H$ corresponding to $X_E$, i.e. it belongs to the factor corresponding to $X_E$. 
%We iteratively update the nodes and edges in $H$ by updating $X_V$ and $X_E$ in $B$, specifically, let $A$ be the adjacency matrix of $B$, and let $W_l$ and $W_r$ be the weights for $X_V$ and $X_E$:
\begin{equation*} X_{E}' =  \sigma(A^TX_V W_E + b_E)
    \textrm{ \hspace{10pt} and \hspace{10pt}} 
   X_{V}' =  \sigma(A  X_E' W_V + b_V)
\end{equation*}
where $\sigma$ is a nonlinear activation function, $W_V, W_E \in \mathbb  R^{d \times d}$ are weight matrices, and $b_V, b_E \in \mathbb R^D$ are bias matrices. 

%We apply layers of this form a fixed number of times to produce the full neural network, and learn $W_V,W_E,b_V,b_E$ through backpropagation. 

We use $\mathcal{N}_i$ throughout to denote the edge neighborhood of $v_i$, i.e. the set of $j$ with $v_i \in e_j$. Similarly, $\mathcal N_j$ denotes the vertex neighborhood of $e_j$, i.e. the set of $i$ with $v_i \in e_j$.
 %denotes hyperedges $\{e_j\}$ in the neighborhood of the hypernode $v_i$, and $\mathcal{N}_j$ hypernodes $\{v_i\}$ in the neighborhood of the hyperedge $e_j$.

%\todo{figure for architecture}

\begin{algorithm}[tb]
   \caption{\hh hypergraph convolution algorithm for node prediction}
   \label{alg-hyper}
\begin{algorithmic}
   %\hspace*{\algorithmicindent}
   \State {\bfseries Input:} Hypergraph incidence matrix $A \in \mathbb{Z}^{n \times m}$, hypernode representations $X_V$ 
   \State {\bfseries Input:} Set of target labels $\{y\}_{k\in L}$
   \State {\bfseries Compute:} $D_{E,l,\alpha}$,  $D_{V ,r, \alpha}$, $D_{V,l,\alpha}$,  and $D_{E ,r , \alpha}$ as in \S~\ref{sec-norm}
   %hypernode degrees $|\mathcal{N}_i|$ and hyperedge cardinalities $|\mathcal{N}_j|$  
   \For{$i=1$ {\bfseries to} \textit{n\_epochs}}
   \State {\bfseries Initialize} $X_E \leftarrow \tilde{0}$. Project $X_V$ to    hidden dimension 
   \For{$j=1$ {\bfseries to} \textit{n\_layers}}
   \State {\bfseries Normalize hypernodes:} $\widetilde{X_V} = D_{E,l,\beta} ^{-1} A D_{V ,r , \beta} X_V$
   \State {\bfseries Update hyperedges:} $X_E =\sigma( \widetilde{X_V}W_{E,j} + b_{E, j}) $
   %\STATE $e_j = \sigma\left(W_{ve}\frac{1}{\sum \exp(\alpha *vwt_i)}\sum_{v_i\in \mathcal{N}_j} \exp(\alpha *vwt_i) v_i + b_e\right)$
   \State {\bfseries Normalize hyperedges:} $\widetilde{X_E} = D_{V,l,\alpha} ^{-1} A D_{E ,r , \alpha}  X_E$
   \State {\bfseries Update hypernodes:} $X_V =\sigma(\widetilde{X_E} W_{V, j}  + b_{V,j}) $
   %\STATE {\bfseries Reassign:} $X_V\leftarrow X_V'$, $X_E\leftarrow X_E'$
   \EndFor
   \State Compute cross-entropy {\bfseries loss} between $\{y\}_{k\in L}$ and predictions using $\{X_{V}\}_{k\in L}$ 
   \State {\bfseries Backpropagate} on loss and {\bfseries optimize } parameters e.g. $\{W_E\}_j$, $\{W_V\}_j$ $\{b_E\}_j$, $\{b_V\}_j$
   \EndFor 
   %\UNTIL{$noChange$ is $true$}
   \State {\bfseries Return:} Learned representations $X_V$, $X_E$ for prediction tasks
\end{algorithmic}
\end{algorithm}
%where $A \in \mathbb{R}^{n\times m}$, $X_V \in \mathbb{R}^{m\times d}$, and $W_l \in \mathbb{R}^{d\times d}$, so the multiplication checks. Here $m=|X_V|$, $n=|X_E|$, and $\sigma$ is a non-linear activation function.

%This gives us updated representations for the nodes and hyperedges in $H$, which are used in updating $W_l$ and $W_r$ through backpropagation.

%Given a hypergraph $H$, we denote the associated bipartite graph as \bh{H}

%We fix some notation:
%\begin{itemize}
 %   \item $v_i \in X_V$, $e_j \in X_E$: vector representations of the hypernodes and hyperedges. 
    %\item $vwt_i$,  $ewt_j$: weights of the hypernodes and hyperedges, currently $vwt_i=|\mathcal{N}_i|$ and $vwt_j=|\mathcal{N}_j|$. 
    %\item $L$ denotes the graph Laplacian (for the bipartite graph), and $\widetilde{L}=D^{-1/2} L D^{-1/2}$ the symmetric normalized Laplacian.
    %\item $X_V$, $X_E$ are vector representations of the left and right hand side nodes of \bh{H} 
%\end{itemize}

%%\subsubsection{graphical models}
%Connection to graphical models approach.
\subsection{Relationship with clique and star expansions} 

In this subsection, we discuss how \hh hypergraph convolution relates to graph convolution. To apply graph convolution to hypergraph problems, we must build a graph $G$ from our hypergraph $H$. There are two main approaches to this in the literature.

The first, the {\em clique expansion} \cite{clique-exp, cliqueZhou, cliqueEmbed, cliqueHyperlink} produces a graph whose vertex set is $V$ by replacing each hyperedge $e=\{v_1,\dots, v_k\}$ with a clique on the vertices $\{v_1,\dots,v_k\}$. A slight variant of this produces a weighted graph where the weight of each edge in the clique is equal to some fixed function of $k$. We will consider another variant $G_c$, where we replace each hyperedge with a clique plus an edge from each vertex $v_1,\dots, v_k$ to itself.

The second, the {\em star expansion} or {\em bipartite graph model} \cite{star-exp}, produces a graph $G_*$ whose vertex set is $V \cup E$, with an edge between a hypernode $v$ and a hyperedge $e$ if $v\in e$, and with no edges between hypernodes and hypernodes or hyperedges and hyperedges, making the graph bipartite. (For $e= \{v_1,\dots, v_k\}$ a hyperedge, the subset $\{e,v_1,\dots, v_k\}$ is a star graph, explaining the name.)

We briefly summarize the relationships between our method and graph convolution on these two graphs: 

\begin{itemize}
    \item If we use the same weights for hypernodes and hyperedges, setting $W_E=W_V$ and $b_E=b_V$, then our method is equivalent to graph convolution on the star expansion $G_*$.
    
    \item If we remove the nonlinear activation function $\sigma$, and consider only a linear map, then our method is equivalent to graph convolution on the clique expansion $G_c$.

\end{itemize}

These claims will be formally justified in Equation \eqref{eq-weight} and Lemma \ref{lem-linear}.

Note here that if we remove both the nonlinear activation function and the weight matrix, so we consider only the linear action of the adjacency matrix, then graph convolution on $G_c$ and $G_*$ are equivalent to each other, because the spectra of the adjacency matrices of $G_c$ and $G_*$ are related by a simple operation (explained after Lemma \ref{lem-linear}).

We now explain the advantages of our method over graph convolution on both the clique and star expansion.

\noindent{\bf Comparison with clique expansion. }The clique expansion is a problematic approach to studying any hypergraph problem because it involves a loss of information - i.e. there can be two distinct hypergraphs on the same vertex set with the same clique expansion. An example is provided by the Fano plane. 

This is the hypergraph $F$ with vertex set $\{1,2,3,4,5,6,7\}$ and the seven hyperedges $\{ \{1,2,3\}, \{ 1,4,5\} , \{ 1, 6, 7\}, \\
\{2,4,6\}, \{2, 5,7\} , \{3,4,7\}, \{3,5,6\} \} $. Because each pair of distinct vertices lies in exactly one hyperedge, the clique expansion of $F$ is the complete graph $K_7$ on seven vertices. 

\begin{figure}
\vspace{-1.1em}
\centering
\begin{minipage}[b]{.2\textwidth}
\begin{subfigure}
         \centering
         \includegraphics[width=\linewidth]{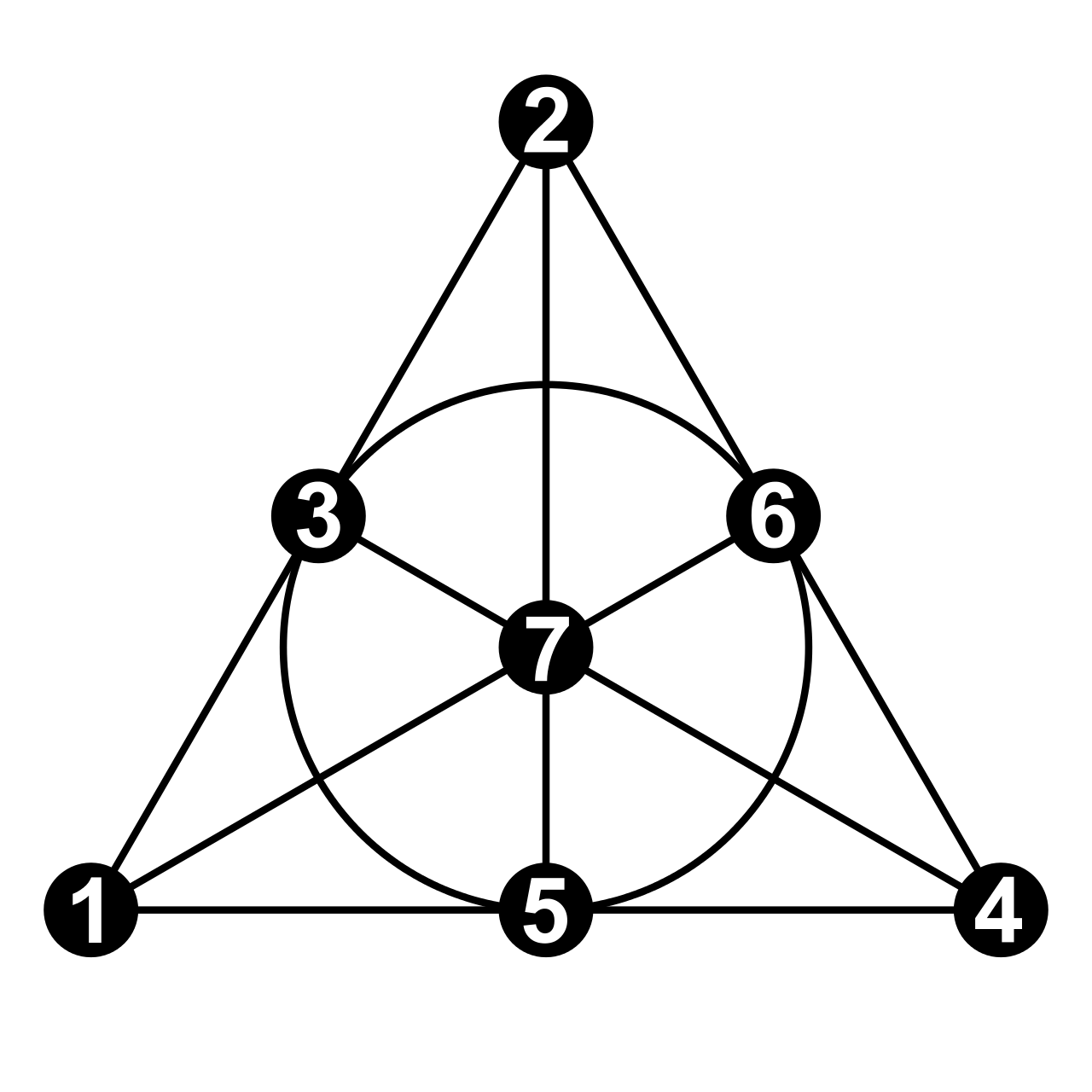}
     \end{subfigure}
     %\vspace{-10pt}
     \end{minipage}\qquad
\begin{minipage}[b]{.2\textwidth}     
     \begin{subfigure}
         \centering
         \includegraphics[width=\linewidth]{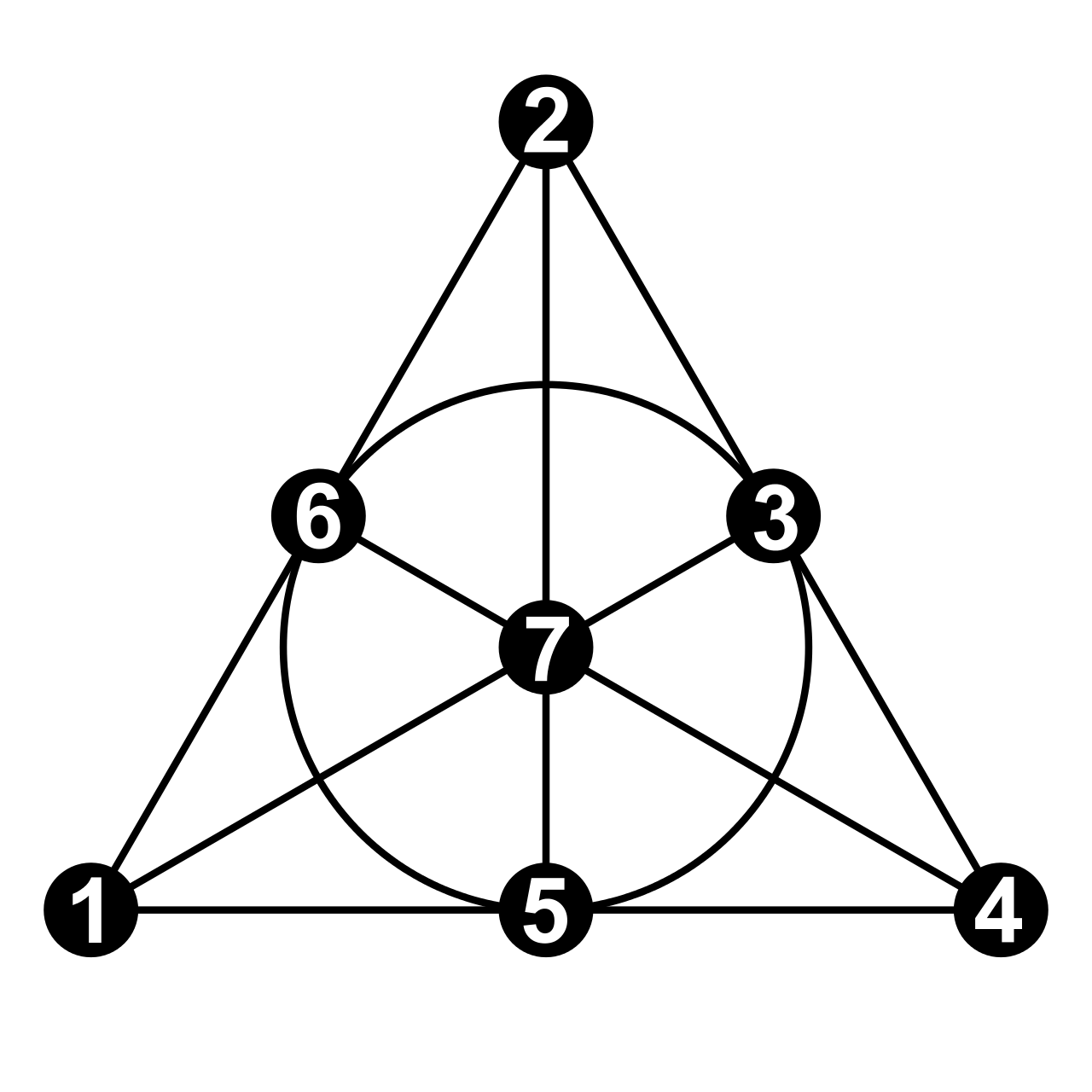}
     \end{subfigure}
     %\vspace{-10pt}
\end{minipage}\qquad
\vspace{-25pt}
\caption{The Fano plane, and a copy with the nodes 3 and 6 permuted. Hypernodes are numbered $1$ through $7$. The six straight lines and the circle represent hyperedges.
When given the hypernode label numbers, or other equivalent information, as input features, approaches based on the clique expansion will not be able to distinguish these two hypergraphs, while HNHN and other approaches can.} \label{fano-figure}
\vspace{-1.1em}
\end{figure}

%\begin{wrapfigure}{r}{.4\textwidth}
%\vspace{-3.5em}
%\includegraphics[width=.2\textwidth]{plots/Fano_plane0.svg.png}
%\caption{The Fano plane. }
%\end{wrapfigure} 

If we permute the seven hypernodes of $F$ by an arbitrary permutation in $S_7$, changing the hyperedges appropriately, we will obtain another hypergraph whose clique expansion is again $K_7$ for the same reason. However it will usually not be the same graph. In fact, by the orbit-stabilizer theorem we can obtain $\frac{7!}{168}=30$ different hypergraphs this way, where $168$ is the number of permutations that stabilize $F$.

Hence graph convolution on the clique expansion will not be able to solve problems on $F$ whose answers change when the hypernode labels are permuted. More generally, examples of this kind can be produced of hypergraphs on $q^2+q+1$ hypernodes for any prime power $q$, using finite projective planes.

\noindent{\bf Comparison with star expansion.} The star expansion does not lose information in this way. However, it treats hypernodes and hyperedges the same. Because, in many datasets, hypernodes and hyperedges describe completely different objects, it is reasonable to give them different weight matrices, and we expect the additional expressive power to improve accuracy.

\cite{holg} showed that clique expansion and star expansion are equivalent in a certain sense, and has been cited to claim that they are equivalent in general. However, this paper worked in a linear setting, considering just the action of the adjacency matrix. In a neural network setting, when a nonlinear activation function $\sigma$ is included, their proof does not apply.

\noindent{\bf Weight simplification.}%\label{ss-weight}
We now explain how removing the distinction between $W_V$ and $W_E$ reduces \hh to graph convolution on $G_*$.

For this purpose, let $B \in \mathbb R^{(n+m) \times (n+m)}$ be the adjacency matrix of $G_*$. Then $$B=  \begin{pmatrix} 0 & A \\ A^T & 0 \end{pmatrix}$$
If we set $W_E=W_V=W$ and $b_E=b_V=b$ in the update rule, we have for all $i$:
$$ X_E^{i+1} = \sigma( A^T X_V^{i} W +b) \textrm{ and } X_V^{i+1} = \sigma ( A X_E^{i+1} W + b)$$
and for all $i$ we define $X^{2i} $ to be the concatenation of $X_V^i$ with the $m \times d$ matrix of zeroes, and $X^{2i+1}$ to be the concatenation of the $n \times d$ matrix of zeroes with $X_E^{i+1}$, then for all $i$ we have
\begin{equation}\label{eq-weight} X^{i+1} =\sigma( B X^i W +b)\end{equation} 
because the matrix $B$ sends the $E$ coordinates to the $V$ coordinates and vice versa.

\noindent{\bf Linear simplification.}\label{ss-linear} To relate graph convolution on $G_c$ with hypergraph convolution, we first relate the adjancency matrix of $G_c$ to the incidence matrix $A$.

\begin{lemma} Let $C$ be the adjacency matrix of $G_c$. Then $C= AA^T$. \end{lemma}

\begin{proof}By definition
$$(A A^T)_{i_1 i_2} =\sum_{j } A_{i_1 j} A_{i_2 j} = \sum_{\substack{j \\ v_{i_1} \in e_j \\ v_{i_2} \in e_j}} 1 $$
and $C_{i_1i_2}$ is the number of edges from $v_{i_1}$ to $v_{I_2}$ in $G_c$. Recall that $G_c$ is the graph obtained from $H$ by replacing each hyperedge with a clique and a union of self-edges. The clique obtained from a hyperedge $e_j$ contains an edge from $v_{i_1}$ to $v_{i_2}$ if and only if $v_{i_1}  \in e_j, v_{i_2}\in e_j,$ and $v_{i_1} \neq v_{i_2} $, and the self-edges obtained from a hyperedge $e$ contain an edge from $v_{i_1}$ to $v_{i_2}$ if and only if $v_{i_1}=v_{i_2} \in e_j$, so in total the number of edges from $v_{i_1}$ to $v_{i_2}$ of $G_c$ is the number of hyperedges containing both $v_{i_1}$ and $v_{i_2}$.

Because each entry of $C$ equals the corresponding entry of $A A^T$, it follows that $C = AA^T$ . \end{proof} 

The next lemma shows that convolution on $H$, without the nonlinear activation function $\sigma$, is equivalent to convolution on $G_c$.

\begin{lemma}\label{lem-linear}

If we define $$X_{E}' = A^T X_V  W_E + b_E\textrm{ \hspace{10pt}  and \hspace{10pt}}X_{V}'= A X_E' W_V + b_V $$ then \begin{equation}\label{no-sigma} X_V'= C X_V W_{c} + b_c\end{equation} where $C$ is the adjacency matrix of $G_c$, $W_c \in \mathbb R^{d\times d }$, and $b_c\in \mathbb R^d$.
\end{lemma}

\begin{proof} We have
$$X_V'=   A X_E' W_V + b_V =   A (A^T X_V  W_E + b_E) W_V + b_V = A A^T X_V W_E W_V + A b_E W_V + b_V.$$
Setting $ W_c= W_E W_V $ and $b_c = A b_E W_V + b_V,$ and using $C = AA^T$, we obtain \eqref{no-sigma}. \end{proof}

%If we define, for all $i$, $$X_{E}^{i+1} = A^T X_V^{i}  W_E^i + b_E^i\textrm{ \hspace{10pt}  and \hspace{10pt} }X_{V}^{i+1}= A X_E^{i+1} W_V^i + b_V^i $$ then \begin{equation}\label{no-sigma} X_V^{i+1} = C X_V^i W_{c}^i + b_c^i\end{equation} where $C$ is the adjacency matrix of $G_c$, $W_c^i \in \mathbb R^{d\times d }$, and $b_c^i \in \mathbb R^d$.
%\end{lemma}

%\begin{proof} We have
%$$X_V^{i+1}=   A X_E^{i+1} W_V^{i} + b_V^{i} =   A (A^T X_V^{i}  W_E^i + b_E^i) W_V^{i} + b_V^{i} = A A^T X_V^i W_E^i W_V^{i} + A b_E^i W_V^{i} + b_V^{i}.$$

%Setting $ W_c^i = W_E^i W_V^{i} $ and $b_c^i = A b_E^i W_V^{i} + b_V^{i},$ and using $C = AA^T$, we obtain \eqref{no-sigma}. \end{proof}

\noindent{\bf Relationship between clique and star expansions.} %\label{ss-compare}
It is a classical result that the eigenvalues of $B$ are equal to the set of square roots of eigenvalues of $A A^T=C$ together with $m-n$ additional zeroes (see e.g. \cite{bfp}). Thus, the spectral theory of $G_*$ is no more complex or simpler than the spectral theory of $G_c$.

\subsection{Additional connections to prior work} %We highlight and draw comparison with some recent work on hypergraph representation learning.
%Several related works have explored learning hypergraph representations %by creating graph out of hypergraph \cite{hgnn, hypergcn, sagnn, hcha}. \todo{more}
Both HGNN \cite{hgnn} and the method of HCHA \cite{hcha} (when the optional attention module is not used) are mathematically equivalent to ordinary graph convolution on the clique expansion. As shown above with the Fano plane, these approaches don't utilize all structural information in the hypergraph. %while translating it to graphs.

HyperGCN \cite{hypergcn} relies on replacing a hypergraph with a graph, but in a more subtle way then a clique expansion. Rather than replacing a hyperedge $\{v_1,\dots,v_k\}$ by a clique, i.e., an edge between every pair of vertices, they pick two special vertices from $\{v_1,\dots,v_k\}$, called \textit{mediators}, and keep only the edges connecting to at least one of the two mediators \cite{ghL}.

This also does not always retain the full information from the hypergraph structure. In fact, in the case where every hyperedge contains at most three hypernodes, HyperGCN is equivalent to the clique expansion, because every edge in the clique touches at least one mediator. As every hyperedge in the Fano plane contains three hypernodes, the Fano plane again shows that this method doesn't use all the structural information. Furthermore, \hh gives a principled way of representing a hyperedge, without devising rules for selecting representatives for each. 

\cite{hypergcn} suggests that hypergraph convolution should be modeled on hypergraph analogues of the graph Laplacian that, unlike the usual Laplacian, are nonlinear. We instead use the (linear) incidence matrix to encode the graph structure and rely on the activation function of the neural network to incorporate nonlinearity. %This approach builds on the strengths of the neural network model, which have been seen in an enormous number of application.
%\hh is computationally efficient, like HyperGCN, \hh creates $O(|N_j|)$ number of edges for each hyperedge $e_j$, in contrast to HGNN that expands hyperedges into cliques, hence requiring $\binom{|N_j|}{2}$ number of edges. 

Another hypergraph learning method, Hyper-SAGNN \cite{sagnn}, uses a very different architecture. Rather than studying the global graph structure, Hyper-SAGNN focuses on a fixed set of vertices to predict whether they are connected by a hyperedge, treating them symmetrically when doing so. 

On the other hand, some prior graph convolution works have taken a similar approach to HNHN by placing neurons on each hyperedge. These include \cite{mgc}, and, bulding on their work, \cite{message-chemistry}. Unlike HNHN, these works restrict to the case of graphs, and they do not use our normalization scheme, discussed below. 

\subsection{Hypergraph normalization}
\label{sec-norm}
One common message passing scheme is to normalize the messages by the cardinality of the neighborhood after message pooling. In the case of passing messages from hyperedges to hypernodes, this can be expressed as
$$X_V' =\sigma(  D_V^{-1} A X_E' W_V + b_V) $$
where $D_V \in \mathbb R^{n \times n}$ is the diagonal matrix defined by
$$ (D_V)_{i_1i_2}= \begin{cases} |\mathcal N_{i_1} |& i_1 = i_2 \\ 0 & i_1 \neq i_2 \end{cases} $$

This ensures, for unbounded activation functions $\sigma$ like ReLU, that the representation vector does not grow rapidly from one layer to the next when $v_i$ has large degree, and, for bounded activation functions $\sigma$, that the input of $\sigma$ is not too large (because bounded activation functions are approximately constant on large values, which leads to vanishing gradient). We could also multiply on the left and right by $D_V^{-1/2}$, which would be equivalent for activation functions like ReLU that commute with multiplication by a positive real number.

However, the contribution of each hyperedge is given the same weight, regarldess of its degree. We consider a generalization where, before summing over hyperedges incident to a given hypernode, we weight the contribution of each hyperedge by a power of its degree, depending on a real parameter $\alpha$:
$$X_V' =\sigma(  D_{V,l,\alpha} ^{-1} A D_{E ,r , \alpha}  X_E' W_V  + b_V) $$ where $D_{E,r,\alpha}\in \mathbb R^{m \times m}$ and $D_{V, l, \alpha} \in \mathbb R^{n \times n}$ are the diagonal matrices given by  $$(D_{E,r,\alpha})_{j_1 j_2} =\begin{cases} |\mathcal N_{j_1} |^\alpha & j_1 = j_2 \\ 0 & j_1 \neq j_2 \end{cases}  \textrm{\hspace{10pt} and \hspace{10pt}}(D_{V, l, \alpha} )_{i_1i_2}= \begin{cases} \sum_{j \in \mathcal N_{i_1}} |\mathcal N_j|^\alpha   & i_1 = i_2 \\ 0 & i_1 \neq i_2 \end{cases} $$
If $\alpha=0$, then $D_{E,r,\alpha}$ is the identity matrix and $D_{V, l,\alpha} =D_v$, so this generalizes the normalization above.

Equivalently, to calculate the $(i,t)^{th}$ entry of $X_V'$, the matrix multiplication above gives the following formula:
$$(X_V')_{it} = \frac{\sum_{j \in \mathcal N_i} |\mathcal N_j|^{\alpha}  \sum_{ s=1}^d  (X_E')_{is} (W_V)_{st} }{ \sum_{j \in \mathcal N_i} |\mathcal N_j|^\alpha } + (b_V)_t $$ Here the $D_{E,r,\alpha}$ matrix produces the factor $|\mathcal N_j|^{\alpha} $ in each term in the numerator. After including $D_{E,r,\alpha}$ we choose the $D_{E, l ,\alpha}$ matrix to produce the factor $\sum_{j \in \mathcal N_i} |\mathcal N_j|^\alpha $ in the denominator, to normalize for the total size of the numerator. 

When $\alpha>0$, the contributions of hyperedges with large degrees are increased compared with the previous normalization, while if $\alpha<0$, their contributions are decreased. We treat $\alpha$ as a hyperparameter to be optimized for a given data set. 

%In other words, this normalization scheme weighs the hyperedges and hypernodes depending on their sizes. For instance, on the coauthorship dataset, setting $\alpha<0$ would weighi a paper as less predictive to an author's classification if it's written by many authors.
%https://arxiv.org/pdf/1710.10903.pdf
By fixing $\alpha=0$, as is equivalent to the normalization in many prior works \cite{gnn-kipf,hcha,hgnn}, we would be assuming that hyperedges of different sizes have equal importance, on average, for learning the representation of vertices. Instead, we believe that on some datasets, hyperedges with a large number of vertices are an ineffective guide to labels, compared to hyperedges with a small number of vertices, while for other datasets it may be the other way around. For instance, in coauthorship graphs, papers with a large number of authors may frequently reflect large collaborations between scientists of different fields, and thus will not be as predictive as papers with a smaller number of authors. By allowing a varying hyperpameter $\alpha$, we can choose a weighting for hyperedges that is appropriate for a given dataset. 

Analogously, we weight vertices in terms of their degrees according to hyperparameter $\beta\in \mathbb{R}$:
$$X_E' =\sigma(  D_{E,l,\beta} ^{-1} A D_{V ,r , \beta}  X_V W_E + b_E) $$ 
where
$$(D_{E,l,\beta})_{j_1 j_2} =\begin{cases} \sum_{ i \in \mathcal N_{j_1}} |\mathcal N_i|^\beta & j_1 = j_2 \\ 0 & j_1 \neq j_2 \end{cases}  \textrm{\hspace{10pt} and \hspace{10pt}}(D_{V, r, \beta} )_{i_1i_2}= \begin{cases}| \mathcal N_{i_1} |^\beta   & i_1 = i_2 \\ 0 & i_1 \neq i_2 \end{cases} $$

\subsection{Implementation details}

As outlined in Algorithm \ref{alg-hyper}, the inputs of our algorithm are the incidence relations between the hypernodes and hyperedges, the input feature vectors of the hypernodes, and the target labels for a subset of the hypernodes. From the hypergraph structure, we extract neighborhood information to compute the normalization factors as in \S\ref{sec-norm}, then each layer of \hh hypergraph convolution relays signals from the hypernodes to the hyperedges, with a nonlinearity and node-specific normalization, then analogously from hyperedges to hypernodes. Finally the loss is calculated using cross entropy between the predicted labels based on the learned node representations and the given target labels. Dropout is added to any convolution layer that's not the last to mitigate overfitting \cite{dropout}. ReLU \cite{relu} is used as the nonlinearity throughout.
%$|\mathcal{N}_j|$ of each hyperedge and degrees of each hypernode $|\mathcal{N}_i|$, used for normalization as discussed in Section~\ref{sec-norm}.
%then each layer of \hh hypergraph convolution relays signals from the hypernodes to the hyperedges, with corresponding normalization and nonlinearities, and signal passing from hyperedges to hypernodes.

While Algorithm~\ref{alg-hyper} is stated in terms of the incidence matrix $A$ for clarity, \hh convolution doesn't require any explicit instantiation of $A$, which has size $O(mn)$. %e.g. $\sim 10^7$ entries for some datasets. 
Even though the number of nonzero entries in $A$ is usually much smaller and thus can be represented as a sparse matrix, never instantiating $A$ is more advantageous for training, as GPUs are much more adaptable for dense over sparse operations \cite{gpu-bottleneck}. %as the latter requires more logical branching. 
Our implementation uses efficient GPU operations that only uses indices of hyperedges connected to a hypernode and of hypernodes connected a hyperedge, without ever instantiating the incidence matrix. This contributes to the fast \hh runtimes as reported in Table~\ref{tab-acc}. 
%We tune hyperparameters such as the number of layers and dropout rate. In practice, at most two convolution layers are needed to achieve the reported accuracies.  

\textbf{Model training} uses the Adam optimizer \cite{adam2015} with a schedule that multiplicatively reduces the learning rate by a factor 0.51 per 100 epochs. While we tune hyperparameters per dataset and task, we found HNHN performance to be robust with respect to hyperparameter variations. The hyperparameters used are: 0.3 dropout rate; 0.04 initial learning rate; 200 training epochs; and the hidden dimension is 400 for all datasets except PubMed, with hidden dimension 800.

\subsection{Time complexity}

 For a hypergraph $H$ with $n$ hypernodes, $m$ hyperedges, average vertex degree $\delta_V$, and hidden dimension $d$, the time complexity of \hh is $O( n \delta_V d + n d^2 + md^2)$. This complexity is because multiplying $X_V$ by $A^T$ or $X_E$ by $A$ takes time proportional to the number of nonzero entries of $A$, which is $n \delta_V$, times $d$. Multiplying by $W_V$ takes $O(nd^2) $ while multiplying by $W_E$ takes time $O(md^2) $. Finally, applying $\sigma$ to each entry takes time $O(nd+md)$ which is $\leq O( nd^2+md^2)$.

Thus, \hh is similar to the most computationally efficient existing hypergraph convolution algorithms, such as HyperGCN. HyperGCN replaces each hyperedge $e_j$ with $2 |\mathcal N_j|-3 = O(|\mathcal N_j|)$ edges, thus replacing $H$ with a graph that has $O( m \delta_E) =O (n \delta_V)$ edges, where $\delta_E$ is the average hyperedge degree of $H$. Applying graph convolution on this graph then takes time proportional to the number of edges, which is again $O (n \delta_V)$.

\hh is faster than hypergraph algorithms based on clique expansion, which require replacing a hyperedge $e_j$ with $\frac{ |\mathcal N_j| (|\mathcal N_j|-1)}{2} = O( |\mathcal N_j|^2) $ edges, for a total of $O(m \delta_E^2) = O( n \delta_V \delta_E) $ edges, producing a graph on which graph convolution takes time $ O( n \delta_V \delta_E d) $.

%\hh is computationally efficient, like HyperGCN, \hh creates $O(|N_j|)$ number of edges for each hyperedge $e_j$, in contrast to HGNN that expands hyperedges into cliques, hence requiring $\binom{|N_j|}{2}$ number of edges. 
Table~\ref{tab-acc} describes the timing results for training node classification. Consistent with the above analysis, algorithms that expand each hyperedge into a linear number of edges perform faster than clique expansion-based algorithms. %\todo{explain why \hh fast }

\section{Experiments}
\label{sec-experiments}
\subsection{Dataset processing} 
We evaluate the quality of \hh hypergraph representations on several commonly used benchmark datasets: CiteSeer \cite{citeseer-getoor} and PubMed \cite{pubmed-data}, both co-citation datasets, Cora \cite{cora-liu, cora-sen} and DBLP \cite{dblp-data}, both co-authorship datasets. Co-authorship data consist of a collection of papers with their authors. %, along with the category each paper belongs to. 
To create a hypergraph, each paper becomes a hypernode and each author a hyperedge. The hypernodes $v_1, ..., v_k$ are connected to a hyperedge $e$ if the papers corresponding to $v_1, ..., v_k$ are written by the author corresponding to $e$. Co-citation data consists of a collection of papers and their citation links. To create a hypergraph, each hypernode represents a paper, the hypernodes $v_1, ..., v_k$ are connected to the hyperedge $e$ if the papers corresponding to $v_1, ..., v_k$ are cited by $e$.

%To create the initial hypernode representation vectors $X_V$, TFIDF representations are created based on contents in the paper. %, for Cora this content is the abstract, and for CiteSeer and DBLP this is the paper content. 
%The vocabulary used for the TFIDF vectors consist of the most common words in each dataset. %The vocabulary sizes and other data information are detailed in the supplementary material. %Table~\ref{tab-data}.
%CiteSeer and PubMed data were obtained from \cite{citeseer-data}, DBLP was processed according to \cite{hypergcn}, and Cora was parsed and processed from the source \cite{cora-data}. Nodes not connected to any hyperedge, as well as hyperedges containing only one hypernode, were removed.

Cora co-authorship data were extracted and parsed from the source \cite{cora-data}. Papers from the top ten subject categories are used, alongside their authorship information and content. Hypernodes not connected to any hyperedge, as well as hyperedges containing only one hypernode, were removed. From the abstracts, TFIDF representations with a vocabulary size of 1000 are created as initial input features. To build the vocabulary, the most frequent one- and two-grams from all abstracts are selected, with a simple stopword removal mechanism: any term that occurs in more than 20\% of documents are discarded. 

DBLP data were collected from \cite{dblp-data-aminer} and processed as described in \cite{hypergcn}: bag-of-words vectors using paper abstracts from five computer science conferences across six subject categories, with the additional step of removing hypernodes not connected to any hyperedge, as our goal is to analyze what can be learned from the hypergraph structure via convolution and message-passing, rather than merely using the node features. Because hypernodes with degree zero have trivial hypergraph structure, and predictions for them would only depend on the features of that node, we remove them from both our training and test sets.

Pubmed and CiteSeer hypergraphs are created using the cocitation relations as given in \cite{citeseer-data}, and as above dangling nodes not attached to any hyperedge are removed. The model input features are TFIDF and bag-of-words vector representations of the document contents for Pubmed and CiteSeer, respectively. Further dataset information are detailed in Table~\ref{tab-data}.

\begin{table}[hb]
\addtolength{\tabcolsep}{-2.pt} 
\centering
\vspace{-3pt}
\begin{tabular}{lclclclclclclc}
\hline
\hline
  & $|V|$ & $|E|$ & \# classes & Avg. edge sz. & Avg. node deg. & Label rate & Feature dim. \\ \hline
  Cora & $16313$ & $7389$ & 10 & 4.26 & 1.93& 5.2\% & 1000 \\
  CiteSeer & $1498$ & $1107$ & 6 & 2.61 & 1.40& 15\% & 3703 \\ 
  DBLP & 41302  & 22363 & 6 & 4.71 & 3.53& 4.2\% & 1425\\
  PubMed & 3840  & 7963 & 3 & 4.35 & 9.02& 2\% & 500\\
  \hline
\end{tabular}
\caption{Further dataset details. $|V|$ denotes number of hypernodes, $|E|$ the number of hyperedges. These datasets were chosen to represent a spectrum with variations in graph sizes, node-to-edge ratio, linkage type (co-authorship vs. co-citation), and numbers of classes.
}
\label{tab-data}
\vspace{-4pt}
\end{table}

\subsection{Hypergraph prediction tasks}

\textbf{Hypernode prediction.} Given a hypergraph and node labels on a small subset of hypernodes, this task is to predict labels on the remaining hypernodes. To train, cross-entropy loss is calculated between the predicted and target labels. We use the Adam \cite{adam2015} optimizer with a learning rate scheduler that multiplicatively reduces the learning rate at regular intervals.

Hyperparameters $\alpha$ and $\beta$ are determined by $5$-fold cross-validation on the training set. %\todo{note} 
We tune hyperparameters such as the number of layers and dropout rate. In practice, at most two convolution layers are needed to achieve the reported accuracies. Experiments are done on a 24-core 2.6 GHz-CPU machine with a Nvidia P40 GPU. As shown in Table~\ref{tab-acc}, \hh outperforms state of the art techniques across datasets, while achieving competitive timing results. 

%\section{Results}
%DBLP: 75.9, 82.3
%Cora large: 50.8, 68.8

\begin{table}[t]
  \addtolength{\tabcolsep}{-5.pt} 
\centering
\begin{tabular}{|c|c|c|c|c|c|c|c|c|c|}
\hline  & 
\multicolumn{4}{|c|}{Accuracy} & \multicolumn{4}{|c|}{Timing}\\
\hline
  & DBLP & Cora & CiteSeer & PubMed & DBLP & Cora & CiteSeer & PubMed \\
  \hline
   HyperGCN & 71.3$\pm$1.2 & 55.0$\pm$.9 & 54.7$\pm$9.8 & 60.0$\pm$10.7 & 563.4$\pm$27.8 & 183.4$\pm$2.7 & 15.6$\pm$.2 & 171.1$\pm$2.8\\ 
   * Fast & 70.5$\pm$14.3 & 45.2$\pm$12.9 & 56.1$\pm$11.2 & 54.4$\pm$10.0 & \textbf{11.5$\pm$.1} & \textbf{2.9$\pm$.1} & \textbf{1.1$\pm$0.} & \textbf{2.5$\pm$.1}\\ 
   HGNN & 77.6$\pm$.4 & 58.2$\pm$.3 & 61.1$\pm$2.2 & 63.3$\pm$2.2 & 802.9$\pm$59.2 & 298.4$\pm$12.2 & 30.5$\pm$.8 & 270.1$\pm$10.5\\ 
   \hh & \textbf{85.1$\pm$.2} & \textbf{63.9$\pm$.8} & \textbf{64.8$\pm$1.6} & \textbf{75.9$\pm$1.5} & 44.2$\pm$1.3 & 13.6$\pm$5.4 & 1.3$\pm$.1 & 26.6$\pm$.4 \\
 \hline
\end{tabular}
\caption{Hypernode classification accuracy and timing results. Accuracies are in \%, timings are measured in seconds. * Fast stands for HyperGCN Fast. %\todo{hgnn numbers } 
}\label{tab-acc}
\vspace{-5pt}
\end{table}

%\subsection{Hyperedge prediction}
%\hh is well-suited for hyperedge predictions. One advantage of \hh is that it produces a vector representation for each hyperedge, as hyperedges have their structurally distinct role in \hh. This is distinct from many prior works that expand out a hyperedge into multiple nodes or select a representative subset of hypernodes for each hyperedge [cite clique , hypergcn ]. For instance, in co-citation data, it is common to include hyperedge papers as a hypernode   [], thus not having a distinct representative    for hyperedges.

%Here we demonstrate the benefits of structurally distinct hyperedge representations   created by \hh for edge prediction. We use the co-citation dataset citesser [PubMed?], where both hypernodes and hyperedges represent papers.  \todo{no one hot}

%For all algorithms tested, the input consists of feature vectors for both the cited and the citing papers, the co-citation relations, and a subset of hyperedge labels in the form of paper categories.  We use the learned representations for the hyperedge papers in each algorithm, for \hh, these are the representations for the hyperedge neurons, for [] and [], we use the learned representations     using  hypernode . We label 5\% of the hyperedges ,   out of total [] hypernodes and []  hyperedges

%\subsubsection{Prediction with reduced dimensions}
\noindent\textbf{Reducing feature dimensions.} We experiment with node prediction after reducing the input feature dimensions using latent semantic analysis \cite{LSA}. As expected, this led to faster training, at the cost of slightly reduced accuracy. On Cora, after the input feature dimension is reduced from $1000$ to $300$, the accuracies are 63.69$\pm$0.58, 56.63$\pm$0.33, and 42.4$\pm$1.58, for \hh, HGNN, and HyperGCN, respectively, with corresponding training times, 10.84$\pm$0.14, 152.34$\pm$1.14, and 182.56$\pm$2.38 seconds. This is a drop of 0.17, 1.52, and 12.6 points in accuracy, respectively, while the running time is 79.5\%, 51.1\%, and 98.8\% that of the non-reduced case. The accuracy advantage of \hh over state of the art methods is preserved in this setting, showing it is also suitable when computational resources are limited.

\noindent{\bf Edge representations learning.}
Compared to approaches that rely on hyperedge expansions, \hh has the advantage in that it produces one dedicated vector representation per hyperedge. This differs from prior works that expand out a hyperedge into multiple nodes or select a representative subset of hypernodes for each hyperedge \cite{hgnn, hypergcn}. Hence \hh learned representations can be easily adapted for downsteam edge-related tasks such as edge prediction. 
Indeed, on the co-citation dataset CiteSeer where both hypernodes and hyperedges represent papers, when given $15$\% of hyperedge labels (and no hypernode labels), the hyperedge classification accuracy is $62.79\pm 1.43$. This is comparable with the $64.76\pm 1.63$ hypernode prediction accuracy (when only $15\%$ of the hypernodes are labeled), reflecting the hyperedge-hypernode symmetry in HNHN.

\noindent\subsection{Effects of normalization parameters}
As discussed in \S\ref{sec-norm}, the normalization parameter $\alpha$ controls how much weight difference exists between hyperedges of different cardinalities. When $\alpha > 0$, larger-sized hyperedges are given greater weight; when $\alpha < 0$, smaller hyperedges are weighted more. Analogously for $\beta$ on weights of hypernodes. Figure~\ref{fig-alpha-beta} demonstrates the effects of $\alpha$ and $\beta$ on CiteSeer node prediction accuracy, showing that the commonly used normalization corresponding to the special case $\alpha=\beta=0$ is not necessarily optimal for node prediction. We note that the effects of $\beta$ on the accuracy shown in Figure~\ref{fig-alpha-beta} are greater than the effects of $\alpha$. This is likely because the vertex degrees (which we raise to the power $\beta$) vary more than the edge degrees (which we raise to the power $\alpha$). Indeed, the vertex degrees on this dataset have a standard deviation of $3.08$, compared to an average degree of $2.07$, for a ratio of $1.49$, while the edge degrees have a standard deviation of $1.67$ with an average degree of $2.8$, for a ratio of only $0.60$.  %A similar trend holds for other values of $\beta$, as well as analogously when varying $\beta$ while fixing $\alpha$.

%%%%%
\begin{figure}
\vspace{-.3em}
\centering
\begin{minipage}[b]{.46\textwidth}
\begin{subfigure}
         \centering
         \includegraphics[width=\linewidth]{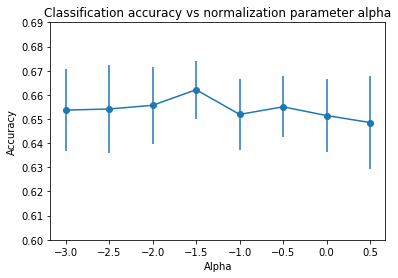}
         %\caption{$y=x$}
         %\label{fig:y equals x}
     \end{subfigure}
     %\vspace{-10pt}
     \end{minipage}\qquad
\begin{minipage}[b]{.46\textwidth}     
     \begin{subfigure}
         \centering
         \includegraphics[width=\linewidth]{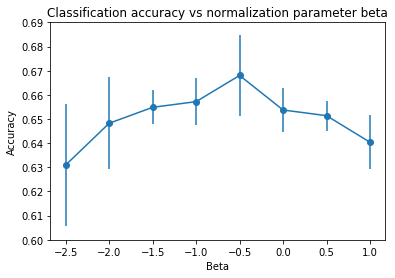}
         %\caption{$y=x$}
         %\label{fig:y equals x}
     \end{subfigure}
     %\vspace{-10pt}
\end{minipage}\qquad
\vspace{-26pt}
\caption{Node classification accuracy on CiteSeer as a function of the normalization parameter $\alpha$  while fixing $\beta=0$ (left) and $\beta$ while fixing $\alpha=0$ (right). This shows normalization has tangible effects on accuracy, and that the commonly used normalization corresponding to $\alpha=\beta=0$ is not always optimal.}
\vspace{-1.1em}
\label{fig-alpha-beta}

\end{figure}
%%%%%

%Our methods in this paper are purely mathematical, not tied to any specific dataset or any direct practical explanation. Thus, to a large extent, it is not possible to predict who may benefit from this research or other consequences.      However, we note that some of the most-used hypergraph datasets consist of information about scientific research (such as papers, authors, and citations), and we study classification problems on these datasets. Thus it is possible that our methods will be used, combined with other techniques, to study practical problems about scientific research, for instance to do with bias. We do not tackle the problem of biases in the dataset, so work on such practical applications will have to consider that aspect. However, our approach should not be any more vulnerable to these biases than any similar methods.

\bibliography{main}
\bibliographystyle{alpha}

\end{document}